\def\year{2022}\relax
\documentclass[letterpaper]{article} 
\usepackage{aaai22}  
\usepackage{times}  
\usepackage{helvet}  
\usepackage{courier}  
\usepackage[hyphens]{url}  
\usepackage{graphicx} 
\usepackage{natbib}  
\usepackage{caption} 
\DeclareCaptionStyle{ruled}{labelfont=normalfont,labelsep=colon,strut=off} 
\usepackage{marvosym}
\usepackage{fontawesome}
\usepackage[ruled,vlined]{algorithm2e}
\usepackage{amsmath}
\usepackage{amssymb}
\usepackage{amsthm}
\usepackage{arydshln}
\usepackage{listings}
\usepackage{subcaption}
\usepackage{tikz}
\usetikzlibrary{arrows,automata,shapes}
\newcommand{\tup}[1]{{\langle #1 \rangle}}
\usepackage{multicol}
\usepackage{multirow}
\usepackage{appendix}

\newcommand{\strips}{\textsc{Strips}}     

\newtheorem{definition}{Definition}
\newtheorem{theorem}{Theorem}

\title{Scaling-up Generalized Planning as Heuristic Search with Landmarks\\ (extended version) }
\author{\Large \textbf{Javier Segovia-Aguas\textsuperscript{\rm 1}, Sergio Jim\'enez\textsuperscript{\rm 2}, Anders Jonsson\textsuperscript{\rm 1} and Laura Sebasti\'a \textsuperscript{\rm 2}}\\ 
}
\affiliations{
\textsuperscript{\rm 1}Universitat Pompeu Fabra \\
\textsuperscript{\rm 2}VRAIN - Valencian Research Institute for Artificial Intelligence, Universitat Polit\`ecnica de Val\`encia \\
javier.segovia@upf.edu, serjice@dsic.upv.es, anders.jonsson@upf.edu, lsebastia@dsic.upv.es
}

\begin{document}

\maketitle

\begin{abstract}
Landmarks are one of the most effective search heuristics for classical planning, but largely ignored in generalized planning. Generalized planning (GP) is usually addressed as a combinatorial search in a given space of algorithmic solutions, where candidate solutions are evaluated w.r.t.~the instances they solve. This type of solution evaluation ignores any sub-goal information that is not explicit in the representation of the planning instances, causing plateaus in the space of candidate generalized plans. Furthermore, node expansion in GP is a run-time bottleneck since it requires evaluating every child node over the entire batch of classical planning instances in a GP problem. In this paper we define a landmark counting heuristic for GP (that considers sub-goal information that is not explicitly represented in the planning instances), and a novel heuristic search algorithm for GP (that we call {\sc PGP}) and that progressively processes subsets of the planning instances of a GP problem. Our two orthogonal contributions are analyzed in an ablation study, showing that both improve the state-of-the-art in {\em GP as heuristic search}, and that both benefit from each other when used in combination.
\end{abstract}

\section{Introduction}
{\em Generalized planning} (GP) addresses the computation of algorithmic solutions that are valid for a set of classical planning instances from a given domain~\cite{Winner03distill:learning,Levesque:GPlanning:IJCAI11,srivastava2011new,Zilberstein:Gplanning:icaps11,hu2011generalized,belle2016foundations,illanes2019generalized,segovia2019computing,frances2021learning}. In the worst case, each classical planning instance may require a completely different solution but in practice, many  planning domains are  known to have polynomial algorithmic solutions~\cite{helmert2006new,fern2011first}. GP is however a challenging computation task;  specifying an algorithmic solution for a set of classical planning instances often requires features that are not explicitly represented in those instances and hence, they must be  discovered~\cite{bonet2021general}. 

GP is typically addressed as a combinatorial search in a space of algorithmic solutions, where candidate solutions are evaluated w.r.t.~the instances they solve. Recently, heuristic search in the solution space of {\em planning programs} for GP has shown to be effective when guided by goal-oriented heuristic functions~\cite{javi:GP:ICAPS21}. However the used heuristics ignore sub-goal information, and often cause large search plateaus. In addition, each candidate solution is evaluated over the entire batch of classical planning instances of the GP problem, increasing the likelihood of search getting stuck in plateaus.

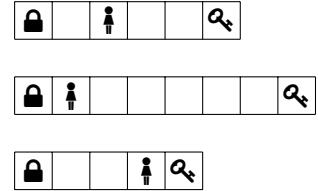
\begin{figure}[t]
    \centering
    \begin{tikzpicture}
    \draw[draw=black,step=0.5cm] (0.0,2.0) grid (3.0,2.5);
    \node at (1.25,2.25) {\LARGE\Ladiesroom};
    \node at (0.25,2.25) {\large\faLock};
    \node at (2.75,2.25) {\faKey};
    \draw[draw=black] (0.0,2) -- (3.0,2);
    
    \draw[draw=black,step=0.5cm] (0.0,1.0) grid (4.0,1.5);
    \node at (0.75,1.25) {\LARGE\Ladiesroom};
    \node at (0.25,1.25) {\large\faLock};
    \node at (3.75,1.25) {\faKey};
    \draw[draw=black] (0.0,1) -- (4.0,1);   

    \draw[draw=black,step=0.5cm] (0.0,0.0) grid (2.5,0.5);
    \node at (1.75,0.25) {\LARGE\Ladiesroom};
    \node at (0.25,0.25) {\large\faLock};
    \node at (2.25,0.25) {\faKey};  
    \end{tikzpicture}
    \caption{\small Three initial states, corresponding to three classical planning instances where an agent must open a lock in a $1\times N$ corridor.}
    \label{fig:example}
\end{figure}    

Figure~\ref{fig:example} shows three initial states that correspond to three classical planning instances where an agent must open a lock in a $1\times N$ corridor. The actions available for the agent are: {\em move} (one cell) right or left, {\em pick-up} or {\em drop} the key, and {\em open} the lock with the key. A generalized plan that solves these three instances, and that generalizes no matter the initial agent location or corridor length, can be formulated as: {\em move right until reaching the end of the corridor, pick up the key, move left until reaching the beginning of the corridor, and finally, open the lock with the key}. Please note that the only information provided by the goals of the previous instances is whether the lock is open. Relevant sub-goal information is however automatically deducible from the problem representation~\cite{hoffmann2004ordered}. For example the {\em landmarks}  in Figure~\ref{fig:lmgraph}, indicating that the agent must reach all cells of the corridor and hold the key, can automatically be extracted from the representation of the first classical planning instance illustrated in Figure~\ref{fig:example}.

\begin{figure}[t]
\centering
\footnotesize
    \begin{tikzpicture}[align=center, node distance=1cm]
    \node [draw] at (0, 0)  (a)    {agent-at(p2)};
    \node [draw] at (-2,0) (b) {lock-at(p0)};
    \node [draw] at (2,0) (c) {key-at(p5)};
    
    \node [draw] at (-1.1,-1) (e) {agent-at(p1)};
    \node [draw] at (1.1,-1) (f) {agent-at(p3)};
    
    \node [draw] at (-1.1,-2) (g) {agent-at(p0)};
    \node [draw] at (1.1,-2) (h) {agent-at(p4)};
    \node [draw] at (1.1,-3) (i) {agent-at(p5)};

    \node [draw] at (2,-4) (j) {agent-has-key()};
    \node [draw] at (-1.1,-4) (k) {unlocked()};
    
    \draw[->] (a) -- (e);
    \draw[->] (a) -- (f);
    \draw[->] (g) -- (k);       
    \draw[->] (e) -- (g);
    \draw[->] (j) -- (k);   
    \draw[->] (b) edge[bend right] (k);  
    \draw[->] (f) -- (h);
    \draw[->] (h) -- (i);
    \draw[->] (i) -- (j);
    \draw[->] (c) edge[bend left] (j);
    \end{tikzpicture}
    \caption{\small Example of a partial {\em landmark graph} for the first classical planning instance illustrated in Figure~\ref{fig:example}.}
    \label{fig:lmgraph}
\end{figure}
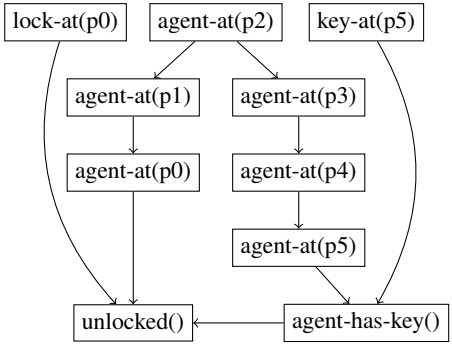    

This paper introduces two orthogonal contributions:
\begin{itemize}
    \item A {\em new heuristic search algorithm for GP}, that improves the performance of computing generalized plans by progressively evaluating them on subsets of the input planning instances.
    \item The adaptation of the {\em landmark graph} from classical planning to GP, and the definition of a {\em landmark counting heuristic} for GP, that considers sub-goal information that is not explicitly represented in the planning instances.
\end{itemize}
The performance of these two orthogonal contributions is analyzed in an ablation study, showing that both outperform {\sc BFGP}, the state-of-the-art in GP as heuristic search~\cite{javi:GP:ICAPS21}, and that both benefit from each other when used in combination.

\section{Background}

\subsection{Classical Planning}
In this work we consider the \strips{} fragment 
of the Planning Domain Definition Language (PDDL) for classical planning~\cite{haslum2019introduction}, that compactly defines a planning problem as $P = \tup{{\cal D},{\cal I}}$, where ${\cal D}$ is the planning {\em domain} and ${\cal I}$ is an {\em instance}. The {\bf domain}\footnote{We assume domain constants to be included in the set of instance objects, and reformulate $\mathcal{A}$ accordingly if required.} ${\cal D}=\tup{{\cal F},{\cal A}}$ consists of a set of FOL {\em predicates}  ${\cal F}$, each of the form $p(x_1,\ldots,x_k)$ where $p$ denotes a $k$-ary predicate symbol and $x_i$, $1\leq i\leq k$, are variables; and a set of {\em action schemes} ${\cal A}$, where each  $\alpha\in{\cal A}$ is defined as $\alpha=\tup{par_\alpha,pre_\alpha,\mathit{eff}_\alpha}$ with $par_\alpha$ denoting its {\em parameters} (arguments), and $pre_\alpha$ and $\mathit{eff}_\alpha$ are sets of atoms defined over variables in $par_\alpha$ that stand for the {\em preconditions} and the {\em effects} of the action schema $\alpha$. The {\bf instance} is defined as ${\cal I} = \tup{\Omega, I, G}$, where $\Omega$ is the finite set of world {\em objects}, $I$ is the {\em initial state}, and $G$ is the {\em goal} condition, a partial state that compactly represents the subset of goals states $S_G$. 

A {\em state} consists of all ground atoms $p(o_1,\ldots,o_k)$, with $k$-ary predicate symbols $p$ and instance objects $o_i\in\Omega$ for $1\leq i\leq k$, interpreted either {\em true} or {\em false}; a {\em partial state} is a subset of all ground atoms. The set of {\em ground actions} $A$ is computed substituting the parameters $par_\alpha$ of each action schema $\alpha\in\mathcal{A}$ with a tuple of objects $\overrightarrow{o}$ of the same size as the action parameters, i.e. a ground action $a\in A$ from an action schema $\alpha\in\mathcal{A}$ is $a = \alpha[\overrightarrow{o}]$ s.t. $|\overrightarrow{o}| = |par_\alpha|$; hence, $pre_a$ and $\mathit{eff}_a$ are partial states after grounding $pre_\alpha$ and $\mathit{eff}_\alpha$ atoms over objects $\overrightarrow{o}$. A ground action $a$ is {\em applicable} iff its preconditions hold in the current state $s$, i.e. iff $pre_a \subseteq s$. Let us first split the action effects into {\em positive} and {\em negative} that respectively interpret ground atoms to true and false after applying action $a$, i.e. $\mathit{eff}_a = \mathit{eff}^+_a\cup \mathit{eff}^-_a$. The successor state $s' = a(s)$ is built removing the negative effects, and then adding the positive action effects, i.e. $s'=(s\setminus \mathit{eff}^-_a)\cup \mathit{eff}^+_a$.

A {\em solution} to $P$ is a sequence of actions, or {\em sequential plan}, $\pi = \tup{a_1,\ldots,a_m}$, such that applied in the initial state $s_0 = I$ it induces a trajectory $\tau = \tup{s_0,a_1,s_1,\ldots,a_m,s_m}$ where each action $a_i(s_{i-1})$ is applicable, and the goal condition holds in the last state, i.e.~$G\subseteq s_m$. 

\subsubsection{Example.} The classical planning instances, with initial states illustrated in Figure~\ref{fig:example}, can be formulated with ground atoms {\tt\small \{lock-at($p_0$), key-at($p_{N-1}$), agent-at($p_i$), agent-has-key, unlocked, adjacent($p_i$,$p_{i+1}$)\}}, where $N$ is the corridor length and $\Omega=\{p_i | 0\leq i < N\}$ is the set of objects representing the different corridor locations, and four action schemes ${\cal A}=$\{{\tt\small move($x_1$,$x_2$)}, {\tt\small pickup-key($x$), drop-key($x$)}, {\tt\small open-lock($x$)}\}. The initial states of the three classical planning instances can then be represented as {\tt\small $I_1$=\{lock-at($p_0$), key-at($p_5$), agent-at($p_2$)\}, $I_2$=\{lock-at($p_0$), key-at($p_7$), agent-at($p_1$)\}, $I_3$=\{lock-at($p_0$), key-at($p_4$), agent-at($p_3$)\}}, completed with the corresponding {\tt\small adjacent($p_i$,$p_{i+1}$)} ground atoms, that are static. The goal condition is the same for the three instances: {\tt\small $G_1=G_2=G_3=\,$\{unlocked\}}.

\subsection{Landmarks in Classical Planning}
{\em Fact landmarks} were introduced for classical planning by \citeauthor{porteous2001extraction}~(\citeyear{porteous2001extraction}) as a subgoaling mechanism, which later was adopted for the first landmark-based heuristic by \citeauthor{richter2008landmarks}~(\citeyear{richter2008landmarks}). In this article we refer to {\em fact landmarks} as {\em landmarks}.

\begin{definition}[Landmark]
\label{def:landmark}
A ground atom $p(o_1,\ldots,o_k)$ is a landmark of a classical planning problem $P$ iff for every sequential plan $\pi = \tup{a_1,\ldots,a_m}$ that solves $P$, the ground atom $p(o_1,\ldots,o_k)$ holds for some state $s_i$ with $0\leq i \leq m$ of the induced trajectory $\tau = \tup{s_0,a_1,s_1,...,a_m,s_m}$.
\end{definition}

According to Definition~\ref{def:landmark}, all the facts appearing in the initial state and goals of a classical planning problem are {\em landmarks} (resp. considering time steps $i=0$ and $i=m$). Landmarks can also be formulae over the state variables and actions. For instance, {\em disjunctive landmarks} indicate that, for every sequential plan that solves $P$, one of the atoms in a given disjunction holds at some state $s_i$, {\small $0\leq i \leq m$}, in the induced trajectory $\tau$. Landmarks can  be (partially) ordered according to the time step where they must be achieved~\cite{hoffmann2004ordered,richter2008landmarks,karpas2009cost}.

In this paper we focus on {\em fact}, {\em disjunctive landmarks}, and their orderings, while considering as future work other formalisms such as {\em action landmarks}~\cite{karpas2009cost,helmert2009landmarks, buchner2021exploiting}. The two kinds of orderings we extract between landmarks are {\em natural orderings},  i.e.~a landmark is true some time before another landmark, and {\em greedy necessary orderings}, i.e.~one landmark is always true one step before another landmark becomes true for the first time~\cite{hoffmann2004ordered}.

Given a classical planning problem $P$, its corresponding {\em landmark graph}, $LG =\tup{LM,O}$, is a directed graph that comprises the set of landmarks $LM$, and the set of orderings $O$ between these landmarks. For instance, Figure~\ref{fig:lmgraph} illustrates the landmark graph of the first instance introduced in Figure~\ref{fig:example}. For clarity, static atoms indicating the adjacency of two corridor cells and natural orderings are omitted.

\subsection{Generalized Planning with Planning Programs}
This work builds on top of the inductive formalism for GP, where a GP problem is a finite set of classical planning instances from a given domain~\cite{jimenez2019review}. In more detail, we build on top of the {\em GP as heuristic search} approach~\cite{javi:GP:ICAPS21}, which represents generalized plans as {\em planning programs}. 
\begin{definition}[GP problem]
    \label{def:gp-problem}
    A {\em GP problem} $\mathcal{P}=\{P_1,\ldots,P_T\}$ is a finite and non-empty set of $T$ classical planning problems  $P_1=\tup{\mathcal{D},\mathcal{I}_1}, \ldots, P_T=\tup{\mathcal{D},\mathcal{I}_T}$ that belong to the same domain $D$, and where each instance $\mathcal{I}_t$,  {\small $1\leq t \leq T$}, may actually differ in the set of ground atoms and actions, initial state, or goals.
\end{definition}

Unlike sequential plans, {\em planning programs} include a control flow construct which allows the compact representation of solutions to classical and GP problems~\cite{segovia2019computing}.  Formally a {\em planning program} is a sequence of $n$ instructions  $\Pi=\tup{w_0,\ldots,w_{n-1}}$, where each instruction $w_i\in \Pi$ is associated with a {\em program line} {\small $0\leq i< n$}, and is either: 
\begin{itemize}
    \item A {\em planning action} $w_i\in A$.
    \item A {\em goto instruction} $w_i=\mathsf{go}(i',y)$, where $i'$ is a program line and $y$ is a proposition.
    \item A {\em termination instruction} $w_i=\mathsf{end}$. The last instruction of a planning program is always a termination instruction, i.e. $w_{n-1}=\mathsf{end}$. 
\end{itemize}

The execution model for a planning program  is a {\em program state} $(s,i)$, i.e.~a pair of a planning state $s\in S$ and program counter $0\leq i<n$. Given a program state $(s,i)$, the execution of a programmed instruction $w_i$ is defined as:
\begin{itemize}
    \item If $w_i\in A$, and $w_i$ is applicable in $s$, the new program state is $(s',i+1)$, where $s'=w_i(s)$ is the {\em successor} state. If $w_i$ is not applicable the new program state is $(s,i+1)$, i.e. the planning state is unmodified. 
    \item If $w_i=\mathsf{go}(i',y)$, the new program state is $(s,i')$ if $y$ holds in $s$, and $(s,i+1)$ otherwise. The proposition $y$ can actually be the result of an arbitrary expression on the state variables, e.g.~a state {\em feature}~\cite{lotinac2016automatic}. 
    \item If $w_i=\mathsf{end}$, program execution terminates. 
\end{itemize}

To execute a planning program $\Pi$ on a classical planning problem $P=\tup{\mathcal{D},\mathcal{I}}$, the initial program state is set to $(I,0)$, i.e.~the initial state of $P$ and the first program line of $\Pi$. A program $\Pi$ {\em solves} $P$ iff the execution terminates in a program state $(s,i)$ that satisfies the goal condition, i.e.~$w_i=\mathsf{end}$ and $G\subseteq s$. Otherwise the execution of the program fails.  The two possible sources of failure of the execution of a planning program $\Pi$ on a classical planning problem $P$ are then:
\begin{enumerate}
\item {\em Incorrect program}, i.e.~execution terminates in a program state $(s,i)$ that does not satisfy the goal condition, i.e.~($w_i=\mathsf{end})\wedge (s\notin S_G)$.
\item {\em Infinite program}, i.e.~execution enters into an infinite loop that never reaches an $\mathsf{end}$ instruction. This can be easily detected whenever a program state is duplicated.
\end{enumerate}

\begin{definition}[GP solution]
    \label{def:gp-solution}
    A {\em generalized plan} $\Pi$ is a solution to a GP problem $\mathcal{P}=\{P_1,\ldots,P_T\}$ iff, for every classical planning problem $P_t\in \mathcal{P}$, $ 1\leq t\leq T$, the sequential plan that results from executing $\Pi$ on $P_t$, i.e. $exec(\Pi,P_t)=\tup{a_1,\ldots,a_m}$, solves $P_t$.
\end{definition}

\section{Planning programs for \strips\ domains}
To build {\em planning programs} that generalize over the instances of a  \strips\ domain, we introduce the notion of {\em pointer} over the world objects, and redefine planning programs accordingly.

\begin{definition}[Pointer] A {\em pointer} $z\in Z$ is a bound variable, with finite domain $D_{z}=[0..|\Omega|)$, that indexes an object of a planning instance ${\cal I}$.
\label{def:pointer}
\end{definition}

We redefine planning programs, so {\em planning actions} $w_i\in A$ are not ground actions over the instance objects, but action schemes ${\cal A}$ instantiated over {\em pointers} in $Z$, i.e. $\alpha[\overrightarrow{z}]\in A_Z$. The execution model of planning programs is updated accordingly;  instructions $w_i=\alpha[\overrightarrow{z}]\in A_Z$ first map every pointer to its indexed object in constant time, which turns the instruction into a ground action $a = \alpha[\overrightarrow{o}]\in A$, from which the standard execution model applies. Figure~\ref{fig:actions} illustrates the relation between (i) an action schema; (ii) its instantiation over pointers; and (iii) its instantiation over objects. Pointers may be typed to address the subset of objects of the same type, although we also refer to them as pointers in the article for short.

\begin{figure}[t]
    \centering
    \begin{tikzpicture}
    \node [draw] at (0, 0)  (a)    {move($x_1$,$x_2$)};
    \node [draw] at (6, 0) (b) {move($z_1$,$z_2$)};
    \node [draw] at (2, -0.8) (d) {move($p_0$,$p_0$)};
    \node [draw] at (2, -1.6) (e) {move($p_0$,$p_1$)};
    \node [draw] at (2, -2.4) (f) {\ldots};
    \node [draw] at (2, -3.2) (g) {move($p_3$,$p_2$)};
    \node [draw] at (2, -4) (h) {move($p_3$,$p_3$)};    
    
    \draw[->, line width=0.5mm] (a) -- (b) node at (3,0.2) {\small $Z$ grounding};
    \draw[dashed] (5.45,-2.2) rectangle (6.95,-0.5);
    \node [draw] at (6.2,-1.0) {$z_1 = 3$};
    \node [draw] at (6.2,-1.7) {$z_2 = 2$};
    \node at (6.2,-2.5) {\small $Z$ indexing in};
    \node at (6.2,-2.8) {\small  current state};
    
    \path[line width=0.5mm] (a) edge (0, -4);
    \draw[->, line width=0.5mm] (0, -0.8) -- (d);
    \draw[->, line width=0.5mm] (0, -1.6) -- (e);
    \draw[->, line width=0.5mm] (0, -2.4) -- (f);
    \draw[->, line width=0.5mm] (0, -3.2) -- (g);
    \draw[->, line width=0.5mm] (0, -4) -- (h);    
    \node [draw=none, rotate=90] at (-0.2,-1.5) {\small $\Omega$ grounding};    
    
    \path[line width=0.5mm] (5.2,-0.3) edge (5.2,-1.7);
    \path[line width=0.5mm] (5.2,-1.0) edge (5.6,-1.0);
    \path[line width=0.5mm] (5.2,-1.7) edge (5.6,-1.7);
    \draw[->,line width=0.5mm] (5.2,-1.7) -- (g);
    \node at (4.2,-2.7) {\small Map $Z\rightarrow\Omega$};
    
    \end{tikzpicture}
    \caption{\small Action schema {\em move}($x_1$,$x_2$) $\in{\cal A}$, where $x_1$ and $x_2$ are free variables, for moving the agent in the domain of Figure~\ref{fig:example}. The action {\em move}($z_1$,$z_2$) $\in{\cal A_Z}$ ground over the pointers $Z=\{z_1,z_2\}$, where $z_1$ and $z_2$ are bound variables in $[0,|\Omega|)$,  that are respectively indexing objects $p_3$ and $p_2$. The set of actions {\footnotesize $\{move(p_0,p_0), move(p_0,p_1), \ldots, move(p_3,p_2), move(p_3,p_3)\}$} ground over the set of objects $\Omega=\{p_0,p_1,p_2,p_3\}$.}
    \label{fig:actions}
\end{figure}
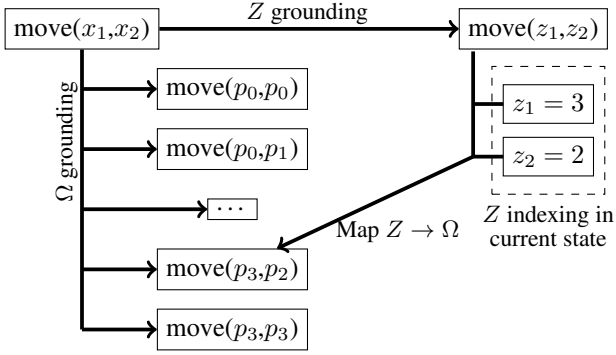

In addition, the set of instructions of planning programs is extended with the set of {\em primitive pointer operations} that comprises: $\{{\tt\small inc}(z_1)$, ${\tt\small dec}(z_1)$, ${\tt\small clear}(z_1)$, ${\tt\small set}(z_1,z_2)$ $| \; z_1,z_2 \in Z\}$ over the pointers in $Z$, and $\{{\tt\small test}_p(\overrightarrow{z}) \; |\; \overrightarrow{z} \in Z^{ar(p)} \}$ over the lists of pointers in $Z^{ar(p)}$ for each predicate symbol  $p\in{\cal F}$ in a given planning domain ${\cal D}$. Respectively these primitive instructions {\em increment}/{\em decrement} a pointer by one, set a pointer to zero, and {\em set} the value of a pointer $z_2$ to another pointer $z_1$.  Instruction ${\tt\small inc}(z_1)$ is applicable iff $z_1 < |\Omega|-1$, and ${\tt\small dec}(z_1)$ is applicable iff $z_1 > 0$, while the remaining instructions are always applicable. The ${\tt\small test}_p(\overrightarrow{z})$ instruction returns the interpretation of $p(\overrightarrow{z})$ at the current state which, similarly to planning actions, requires first to map $\overrightarrow{z}$ to the corresponding indexed objects $\overrightarrow{o}$ s.t. $p(\overrightarrow{o})$ is a ground atom. 

Last the {\em goto} instructions of planning programs are restricted to be conditioned by a single Boolean $y_z$ that, playing the role of a {\em zero} FLAGS register~\cite{dandamudi2005installing}, is dedicated to store the outcome of the last executed primitive operation over pointers. The FLAG $y_z$ allows to keeping the solution space tractable. Formally, it is defined as:
\begin{center}
$y_z \equiv 
\begin{cases} 
      False, & \text{if applicable } w_i = inc(z_1), \\
      (z_1 == 1), & \text{if applicable } w_i = dec(z_1),\\
      True, & \text{if } w_i = clear(z_1), \\
      (z_2 == 0), & \text{if } w_i = set(z_1,z_2),\\
      \neg p(\overrightarrow{z}), & \text{if } w_i = test_p(\overrightarrow{z}), \\
      True, & \text{if inapplicable } w_i. 
\end{cases}
$
\end{center}

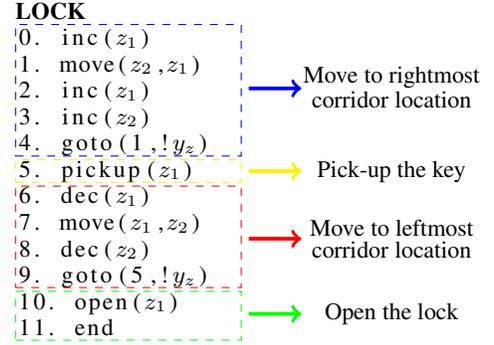
\begin{figure}[t]
\centering
    \footnotesize
        \begin{subfigure}[t]{0.30\textwidth}
            \begin{lstlisting}[mathescape]
$\textbf{LOCK}$
0. inc($z_1$)
1. move($z_2$,$z_1$)
2. inc($z_1$)
3. inc($z_2$)
4. goto(1,!$y_z$)
5. pickup($z_1$)
6. dec($z_1$)
7. move($z_1$,$z_2$)
8. dec($z_2$)
9. goto(5,!$y_z$)
10. open($z_1$)
11. end
            \end{lstlisting}
            \begin{tikzpicture}[remember picture, overlay]
            \draw[dashed,color=blue] (0,2.95) rectangle (3,4.7);
            \draw[dashed,color=yellow] (0,2.6) rectangle (3,2.9);
            \draw[dashed,color=red] (0,1.2) rectangle (3,2.55);
            \draw[dashed,color=green] (0,0.5) rectangle (3,1.15);
            \node at (5,4) {Move to rightmost};
            \node at (5,3.7) {corridor location};
            \node at (5,2.75) {Pick-up the key};
            \node at (5,2) {Move to leftmost};
            \node at (5,1.7) {corridor location};
            \node at (5,0.85) {Open the lock};
            \draw[->, line width=0.5mm, color=blue] (3.1,3.85) -- (3.8,3.85);
            \draw[->, line width=0.5mm, color=yellow] (3.1,2.75) -- (3.8,2.75);
            \draw[->, line width=0.5mm, color=red] (3.1,1.85) -- (3.8,1.85);
            \draw[->, line width=0.5mm, color=green] (3.1,0.85) -- (3.8,0.85);
            \end{tikzpicture}
        \end{subfigure} 
    \caption{\small Example of a planning program that solves the GP problem of the three classical planning instances illustrated in Figure~\ref{fig:example}. Note that $y_z$ captures the outcome of the last executed primitive pointer instruction, and that a pointer increase makes $y_z=False$ until it becomes inapplicable where $y_z=True$, while a decrease instruction makes $y_z=True$ when the pointer decreases from $1$ to $0$ or if it is inapplicable, otherwise $y_z=False$.}
    \label{fig:lock-gplan}  
\end{figure}

{\em Object typing} is naturally supported by pointers, specializing them to the number of objects of a particular type. Pointers can trivially be mapped to a mutually exclusive set of PDDL propositional variables. Likewise primitive pointer operations, and goto instructions, can be coded as PDDL actions with conditional effects.

\subsubsection{Example.} Figure~\ref{fig:lock-gplan} shows a planning program with $n=12$ program lines, and pointers $Z=\{z_1,z_2\}$, that solves the three planning instances above, and that is computed by our PGP algorithm.  Program lines 1, 5, 7 and 10 contain planning actions in $A_Z$, where pointers index the corridor locations. The program lines 4 and 9 contain {\em goto} instructions $A_{\sf go}$ that branch the program execution flow according to the value of the zero flag $y_z$. The remaining lines contain primitive pointers  instructions, that operate over pointers, and update the zero flag $y_z$ accordingly. The last instruction is an ${\tt\small end}$ instruction. The program of Figure~\ref{fig:lock-gplan} leverages the fact that, in the three classical planning  instances given as  input, adjacent locations are named with consecutive numbers.

Pointers are always initialized to zero. Therefore, the execution of the planning program of Figure~\ref{fig:lock-gplan} on
the first classical planning instance illustrated in Figure~\ref{fig:example} produces the following sequential plan $\pi=\langle move(p_0,p_1), move(p_1,p_2), move(p_2,p_3), move(p_3,p_4),$\\ $pickup(p_4), move(p_4,p_3), move(p_3,p_2), move(p_2,p_1),$ \\$move(p_1,p_0), open(p_0)\rangle$. Next we detail the execution of the first five program lines, which produces the sequence of ground actions to reach the rightmost location of the corridor; pointers $z_1$ and $z_2$ are initialized to zero, so they initially index the same object $p_0$. Program line 0 increments the value of pointer $z_1$, so it indexes $p_1$ and hence, the execution of program line 1 corresponds to the execution of the ground action $move(p_0,p_1)$. Lines 2 and 3 increment the two pointers, so $z_1$ indexes $p_2$ while $z_2$ indexes $p_1$. Line 4 indicates that the block of program lines [1$-$4] is repeated until pointer $z_2$ can no longer be incremented. 

\begin{algorithm}[t]
\footnotesize
\SetAlgoLined
\KwData{A GP problem $\mathcal{P}_{n,Z}$}
\KwResult{A planning program $\Pi$ that solves  $\mathcal{P}_{n,Z}$ (or {\em unsolvable})}
 $active \leftarrow \{P_1\}$ \;
 $open$ $\leftarrow$ insertProgram($\emptyset, \Pi_{empty}, active$)\;
 \While{$open\neq\emptyset$}{
  $\Pi\leftarrow$ extractBestProgram($open$) \;
  $childs \leftarrow$ expandProgram($\Pi$, ${\cal P}_{n,Z}$) \;
  \For{$\Pi'\in childs$}{
      \uIf{isSolution($\Pi'$,$active$)}{
        \uIf{isSolution($\Pi'$,${\cal P}_{n,Z}$)}{
            return $\Pi'$;
        }
        $P_{fail} \leftarrow$ getFirstFailed($\Pi'$,${\cal P}_{n,Z}$)\;
        $active \leftarrow active \cup \{P_{fail}\}$\;
        reevaluateQueue($open$,$active$)\;
      }
      \uIf{not isDeadEnd($\Pi',active$)}{
       $open$ $\leftarrow$ insertProgram($open, \Pi', active$)\;
      }
  }
 }
 return $unsolvable$;
 \caption{PGP}
 \label{alg:PGP}
\end{algorithm}

\section{Progressive Generalized Planning}
This section describes our {\em Progressive heuristic search algorithm for Generalized Planning} ({\sc PGP}). This algorithm adapts a {\em Best-First Search} (BFS), in the solution space of the possible planning programs that can be built with $n$ program lines and $|Z|$ pointers, so that it progressively processes the full batch of classical planning instances of a GP problem.  

\subsection{Progressive Best-First Search for GP}
The input to our {\sc PGP} algorithm is a GP problem $\mathcal{P}_{n,Z}=\{P_1,\ldots,P_T\}$. {\sc PGP} outputs a planning program $\Pi$ that solves $\{P_1,\ldots,P_T\}$, or it reports unsolvability within the maximum number of $n$ program lines and $|Z|$ pointers.  Briefly, {\sc PGP} keeps a subset of the classical planning instances called the {\em active instances}, that initially contains only the first classical planning instance of the GP problem. When {\sc PGP} finds a program that solves the full set of {\em active instances}, it validates that program on the remaining instances of the GP problem, and augments the set of {\em active instances} with the first instance for which the program fails. The procedure is repeated until {\sc PGP} finds a program that solves all the instances in the GP problem. PGP can be understood as a variant of {\em counterexample-guided search}~\cite{seipp2018counterexample}.

Algorithm~\ref{alg:PGP} shows the pseudo-code of {\sc PGP}. In more detail, in Lines 1-2, the algorithm {\bf initializes} the subset of {\em active} instances with the first planning instance $P_1\in{\cal P}_{n,Z}$, and
inserts an empty planning program $\Pi_{empty}$ with $n$ undefined program lines, into the empty $open$ priority queue. The main loop runs in Lines 3-16, while there are programs to expand, and it returns a GP solution if found (Line 9), otherwise it returns $unsolvable$ on Line 17. In every iteration, the best-evaluated program $\Pi$ is {\bf selected} and removed from the $open$ priority queue (Line 4), and {\bf expanded} (Line 5). For each child node $\Pi'$, PGP checks whether $\Pi'$ solves the $active$ instances (Line 7). If $\Pi'$ also solves ${\cal P}_{n,Z}$ then it is returned as the GP solution (Lines 8-9), otherwise the first instance $P_{fail}$ where $\Pi'$ fails is added to the $active$ set, and the $open$ queue is {\bf reevaluated} (Lines 10-12) w.r.t. the problems in the augmented {\em active} set. In Lines 13-14, the child $\Pi'$ is evaluated over the set of $active$ instances and {\bf inserted} in the $open$ queue if it is not a {\bf deadend} for the $active$ set.

{\sc PGP} is a {\em frontier search} algorithm meaning that, to reduce memory requirements, it stores only the {\em open list} of generated nodes but not the {\em closed list} of expanded nodes~\cite{korf2005frontier}. With regard to node expansion, let $\Pi$ be the partially specified program that corresponds to the best-evaluated node extracted by {\sc PGP} from the open list (Line 4). {\sc PGP} generates one successor for each program that results from programming the first undefined line of $\Pi$. This node expansion procedure guarantees that duplicate successors are not generated and it keeps the branching factor tractable; at a given undefined program line {\sc PGP} can only program (i) a planning action;  (ii) a primitive pointer instruction; or (iii) a $\mathsf{goto}$ instruction. The maximum depth of the {\sc PGP} search tree is the number of program lines $n$, since only an undefined line can be programmed.

\begin{theorem}[Termination]
The execution of {\sc PGP} on a GP problem, always terminates.
\end{theorem}

\begin{proof}
The only possible cause for non-termination would be that {\sc PGP} could generate duplicate search nodes, allowing the infinite re-opening of an already discarded program. By the definition of the {\sc PGP} expansion procedure, it does not generate duplicate successors; a child node always has one more line programmed than its parent. 
\end{proof}

\begin{theorem}[Completeness]
Given a GP problem, and a maximum number of program lines $n$ and pointers $|Z|$, if there is a solution planning program  within these bounds then PGP can compute it.
\end{theorem}

\begin{proof}
{\sc PGP} only discards a search node when its corresponding partially specified planning program fails to solve a planning instance. This is precisely the definition for not being a GP solution. Further, any planning program that could be built programming the remaining undefined program lines would also fail to solve that same instance.
\end{proof}

\begin{theorem}[Soundness]
If the execution of {\sc PGP} on a GP problem outputs a generalized plan $\Pi$, then this means that $\Pi$ is a solution for that GP problem.
\end{theorem}

\begin{proof}
{\sc PGP} runs until (i) the open list is empty, which means that search space is exhausted without finding a solution and no generalized plan is output; or (ii) {\sc PGP} extracted from the open list a planning program whose execution solves all the instances $P_t\in\mathcal{P}_{n,Z}$. This is precisely the definition of a solution for a GP problem.
\end{proof}

\section{Landmarks in Generalized Planning}
This section defines a landmark counting heuristic for guiding a combinatorial search in our GP solution-space.

\subsection{The landmark graph with pointer assignments}
For each classical planning problem $P_t\in \mathcal{P}_{n,Z}$, we enrich its corresponding landmark graph $LG_t$ with: (i) {\em pointer landmarks}, that indicate pointer assignments that must be achieved by a solution to $P_t$; and (ii) orderings between {\em pointer landmarks} and the regular landmarks, computed by the {\sc LAMA} algorithm. 

First we compute, for every classical planning problem $P_t\in \mathcal{P}_{n,Z}$, {\small $1\leq t \leq T$}, its corresponding {\em landmark graph}  $LG_t=\tup{LM_t,O_t}$, as a pre-processing step. We implement the back-chaining {\sc LAMA} algorithm for finding landmarks and orderings between landmarks~\cite{richter2010lama}. Briefly we start from a set of known landmarks, and find new landmarks that hold in any plan before an already known landmark may become true. The procedure stops when no more new landmarks are discovered. In more detail, given a landmark $q$, if all actions that achieve $q$ for the first time share a precondition $p$, this means that $p$ is also a landmark, and that there is a {\em greedy necessary} ordering $p \rightarrow_{gn} q$ between them. The algorithm  discovers {\em disjunctive landmarks} too; when $q$ is a landmark, and all actions that first achieve $q$ have $p_1$, or $p_2$, \ldots, or $p_k$ as a precondition, this means that  $p_1\vee p_2\vee\ldots\vee p_k$ is also a landmark. We also implement the algorithm for adding {\em additional orderings from restricted Relaxed Planning Graphs} (RPGs) to discover  natural orderings between landmarks~\cite{richter2010lama}.

\begin{definition}[Pointer landmarks]
\label{def:plandmark}
Given a classical planning problem $P$ and a set of pointers $Z$, we say that the assignments $\bigwedge_j\bigvee_i (z_i=o_j)$ are {\em pointer landmarks} iff the ground atom $p(\overrightarrow{o})$ is a landmark for $P$, $z_i\in Z$, and $o_j\in \overrightarrow{o}$.
\end{definition}

In other words given a landmark $p(\overrightarrow{o})$, a {\em pointer landmark} indicates that at least one pointer must point to each object in $\overrightarrow{o}$. {\em Pointer landmarks}, and their corresponding orderings, are computed as follows. For every {\em greedy necessary} order $p \rightarrow_{gn} q$ in the landmark graph $LG_t$ computed by the LAMA algorithm, if an action $a\equiv \alpha(\overrightarrow{o})$ is an achiever for $q$, then we have that the assignments $\bigwedge_j\bigvee_i (z_i=o_j)$ are pointer landmarks, and furthermore, $\bigwedge_j\bigvee_i  (z_i=o_j)\rightarrow_{gn} q$ is a {\em greedy necessary} ordering where $\{z_i\} \subseteq Z$ are the pointers of the same type of the corresponding object $o_j\in\overrightarrow{o}$. For example given the first instance illustrated in Figure~\ref{fig:example} and the set of pointers $Z=\{z_1,z_2\}$, then the corresponding {\em landmark graph} of Figure~\ref{fig:lmgraph} is enriched with the disjunctive landmark $(z_1=5\vee z_2=5)$ and the greedy necessary order $(z_1=5\vee z_2=5)\rightarrow_{gn}$ {\small\tt agent-has-key}, among other orderings. These particular pointer landmarks, and their {\em greedy necessary} ordering, are created since the actions {\small\tt pickup}$(z_1)$ or {\small\tt pickup}$(z_2)$ for $z_1=5$ or $z_2=5$, are first achievers of the {\small\tt agent-has-key} landmark. 

\subsection{The landmark counting heuristic for GP}
The  {\em landmark graph} extracted from a given classical planning instance can be used to guide a forward search in the space of  states reachable from the initial state of that instance~\cite{richter2010lama,buchner2021exploiting}. For example, to implement the evaluation function $f_{LM}(s,\pi)$ of the {\sc LAMA} planner, that computes the number of landmarks that have not been achieved on the path from the initial state to the state $s\in S$, given by the sequence of planning actions $\pi$.  This evaluation function is formalized as: 
\begin{displaymath}
f_{LM}(s,\pi)=|(LM\setminus Reached(s,\pi))\cup ReqAgain(s,\pi)|,
\end{displaymath}
where $LM$ is the set of landmarks discovered with the previously described mechanism, $Reached(s,\pi)\subseteq LM$ is the subset of reached landmarks; a landmark $p$ is first reached in state $s$ if all predecessors of $p$ in the corresponding landmark graph have been reached, and $p\in s$. Once a landmark is reached in a state $s$, it remains reached in all successor states. Last,  $ReqAgain(s,\pi)\subseteq Reached(s,\pi)$ is the subset of landmarks that must be achieved again that comprises: goals $p$ that are false in $s$, and greedy necessary predecessors $p$ of some landmark $q$ that has not been reached yet. Note that $f_{LM}(s,\pi)$ is not a heuristic function in the standard sense, since its value is path-dependent, but it works well for classical planning.

Next we show how we leverage landmarks to inform a combinatorial search in our GP solution-space.  When a program execution terminates because an unspecified program line is reached, we retrieve the last state reached, and estimate how far this state is from a goal state with the following heuristic function:

\begin{itemize}
    \item $f_{LM}(\Pi,{\cal P}_{n,Z})=\sum_t f_{LM}(\Pi,P_t)$ for each $P_t\in{\cal P}_{n,Z}$, where $f_{LM}(\Pi,P_t) =f_{LM}(s,\pi)$ such that: 
\begin{itemize}
\item $\pi = exec(\Pi,P_t)$ is the sequential plan that results from executing $\Pi$ on $P_t$,
\item  $s$ is the last state reached after executing $\Pi$ on $P_t$,
\item $f_{LM}(s,\pi)$ is the landmark counting heuristic defined above.
\end{itemize}    
\end{itemize}

Note that, when used in combination with our PGP search algorithm, the $f_{LM}(\Pi,{\cal P}_{n,Z})$ heuristic is  evaluated only over the $active$ set, i.e. $f_{LM}(\Pi,active)$, instead of over the full GP problem ${\cal P}_{n,Z}$ which allows to saving computations.

A potential issue with our landmark counting heuristic  is that aggregating each $f_{LM}(\Pi,P_t)$ could make instances with more landmarks bias the search. This issue could be mitigated normalizing the heuristic values with the total number of landmarks; however, we have not observed this issue to have an impact in the experiments. 

\begin{table*}[t!]
\scriptsize
    \centering
    \resizebox{\textwidth}{!}{
    \begin{tabular}{|l|c||c|c|c|c||c|c|c|c||c|c|c|c||c|c|c|c|}\hline 
         \multirow{2}{*}{} & \multirow{2}{*}{$n$, $|Z|$} & \multicolumn{4}{c||}{Baseline - BFS$(f_{GC})$} & \multicolumn{4}{c||}{Contribution 1 - BFS($f_{LM})$} & \multicolumn{4}{c||}{Contribution 2 - PGP$(f_{GC})$}& \multicolumn{4}{c|}{Contributions 1\&2 - PGP$(f_{LM})$}   \\\cline{3-18}
         & & T & M & Ex & Ev & T & M & Ex & Ev & T & M & Ex & Ev & T & M & Ex & Ev \\\hline 
         Baking & 13, 6 & TE & TE & TE & TE & TE & TE & TE & TE & ME & ME & ME & ME & {\bf 98} & {\bf 46} & {\bf 48K} & {\bf 1.4M}\\
         Corridor & 11, 2 & 101 & {\bf 27} & {\bf 6K} & 120K () & 45 & 34 & 7K & {\bf 114K} & {\bf 12} & 50 & {\bf 6K} & 120K & 16 & 57 & {\bf 6K} & 121K  \\
         Gripper & 8, 4 & 5 & {\bf 15} & {\bf 3K} & {\bf 63K} & 43 & 20 & 19K & 339K & {\bf 1} & 28 & {\bf 3K} & {\bf 63K} & 8 & 20 & 19K & 336K \\
         Intrusion & 9, 1 & 94 & 284 & 73K & 1.3M & {\bf 0} & {\bf 18} & {\bf 8} & {\bf 190} & 24 & 877 & 74K & 1.3M & {\bf 0} & {\bf 18} & {\bf 8} & {\bf 190}   \\ 
         Lock & 12, 2 & TE & TE & TE & TE & TE & TE & TE & TE & ME & ME & ME & ME & {\bf 3} & {\bf 46} & {\bf 1K} & {\bf 30K} \\
         Ontable & 11, 3 & TE & TE & TE & TE & TE & TE & TE & TE & {\bf 25} & 196 & 10K & 366K & 323 & {\bf 98} & {\bf 3K} & {\bf 113K} \\
         Spanner & 12, 5 & TE & TE & TE & TE & TE & TE & TE & TE & TE & TE & TE & TE  & {\bf 172} & {\bf 604} & {\bf 25K} & {\bf 705K} \\
         Visitall & 7, 2  & {\bf 0} & {\bf 7} & 50 & 511 & 7 & 17 & {\bf 27} & {\bf 249} & {\bf 0} & {\bf 7} & 88 & 900 & {\bf 0} & 17 & 49 & 446  \\\hline
    \end{tabular}
    }
    \caption{\small Number of program lines $n$ and pointers $|Z|$, total time (T) in seconds, memory peak (M) in MB, expanded (Ex) and evaluated (Ev) nodes (K is $10^3$ and M is $10^6$). TE and ME stand for time and memory exceeded. Best results in bold.}
    \label{tab:synthesis}
\end{table*}

\section{Evaluation}
We compare on eight STRIPS domains the performance of PGP($f_{GC}$), our algorithm guided by our landmark counting heuristic for GP, w.r.t the {\em GP as heuristic search} approach~\cite{javi:GP:ICAPS21} that serves as a baseline. We also report an ablation study of our two orthogonal contributions. All experiments were performed using $10$ random input instances of increasing difficulty per domain, in an Ubuntu 20.04 LTS, with AMD® Ryzen 7 3700x 8-core processor and 32GB of RAM, with a 1 hour time bound \footnote{Repository at \url{https://github.com/aig-upf/pgp-landmarks}.}.

\begin{figure}[t!]
    \begin{scriptsize}
\begin{subfigure}[t]{0.2\textwidth}
    \begin{lstlisting}[mathescape]
$\textbf{GRIPPER}$
0. pick($z_b$,$z_{r1}$,$z_g$)
1. inc($z_{r2}$)
2. move($z_{r1}$,$z_{r2}$)
3. drop($z_b$,$z_{r2}$,$z_g$)
4. move($z_{r2}$,$z_{r1}$)
5. inc($z_b$)
6. goto(0,$\neg y_z$)
7. end 

$\textbf{INTRUSION}$
0. recon($z_h$)
1. break-into($z_h$)
2. clean($z_h$)
3. gain-root($z_h$)
4. download-files($z_h$)
5. steal-data($z_h$)
6. inc($z_h$)
7. goto(0,$\neg y_z$)
8. end       

$\textbf{ONTABLE (BLOCKS)}$
0. unstack($z_{o1}$,$z_{o2}$)
1. inc($z_{o2}$)
2. goto(0,$\neg y_z$)
3. put-down($z_{o1}$)
4. clear($z_{o2}$)
5. inc($z_{o1}$)
6. goto(0,$\neg y_z$)
7. clear($z_{o1}$)
8. inc($z_{o3}$)
9. goto(0,$\neg y_z$)
10. end


$\textbf{VISITALL}$
0. visit($z_i$,$z_j$)
1. inc($z_i$)
2. goto(0,$\neg y_z $)
3. clear($z_i$)
4. inc($z_j$)
5. goto(0,$\neg y_z$)
6. end
    \end{lstlisting}
        \end{subfigure} 
\begin{subfigure}[t]{0.2\textwidth}    
    \begin{lstlisting}[mathescape]    
$\textbf{BAKING}$
0. putegginpan($z_e$,$z_p$)
1. putflourinpan($z_f$,$z_p$)
2. mix($z_e$,$z_f$,$z_p$)
3. putpaninoven($z_p$,$z_o$)
4. bakecake($z_o$,$z_p$,$z_c$)
5. removepanfromoven($z_p$,$z_o$)
6. cleanpan($z_p$,$z_s$)
7. inc($z_c$)
8. inc($z_e$)
9. inc($z_f$)
10. inc($z_s$)
11. goto(0,$\neg y_z$)
12. end 

$\textbf{CORRIDOR}$
0. inc($z_1$)
1. move($z_2$,$z_1$)
2. inc($z_1$)
3. inc($z_2$)
4. goto(1,$\neg y_z$)
5. move($z_1$,$z_2$)
6. set($z_1$,$z_2$)
7. dec($z_2$)
8. test(goal-at($z_1$))
9. goto(4,$y_z$)
10. end

$\textbf{SPANNER}$
0. pickup_spanner($z_{l1}$,$z_s$,$z_m$)
1. tighten_nut($z_{l1}$,$z_s$,$z_m$,$z_n$)
2. inc($z_n$)
3. inc($z_s$)
4. goto(0,$\neg y_z$)
5. inc($z_{l1}$)
6. walk($z_{l2}$,$z_{l1}$,$z_m$)
7. clear($z_n$)
8. clear($z_s$)
9. inc($z_{l1}$)
10. goto(0,$\neg y_z$)
11. end
    \end{lstlisting} 
        \end{subfigure}   
    \end{scriptsize}
    \caption{\small Generalized plans computed by $PGP(f_{LM})$. The parameters of non-goto instructions are {\em pointers}. Goto instructions have two parameters; the destination line, and a condition over the zero flag (either $y_z$ or $\neg y_z$).}
    \label{fig:synthesis}
\end{figure}

\subsubsection{Benchmarks.} The {\em Baking} domain, where an agent follows the steps to bake a set of cakes. {\em Corridor}, where an agent moves from an arbitrary initial location to a destination location in a corridor. {\em Gripper}, where a robot must pick all balls from room A and drop them in room B.  {\em Intrusion}, where an attacker performs a number of actions to steal data from some host computers. {\em Lock}, the domain illustrated in Figure~\ref{fig:example}. {\em Ontable}, in which towers of blocks are placed on the table.  {\em Spanner}, where an agent must pick up spanners to tighten the loose nuts at the end of a corridor (spanners can only be used once and the agent cannot go back, introducing dead-ends). {\em Visitall}, where starting from the bottom-left corner of a square grid, an agent must visit all grid cells. 

\subsubsection{Synthesis and ablation study.} 
In the first experiment we use as a baseline the {\em GP as heuristic search} approach~\cite{javi:GP:ICAPS21} that implements a Best-First Search (BFS) guided by its best single heuristic; a Euclidean distance (which actually acts as a {\em counter of unachieved goals} in propositional domains) and that we denote as $f_{GC}$. Table~\ref{tab:synthesis} reports the best solutions found in terms of the number of required program lines $n$ and pointers $|Z|$. The results show that BFS$(f_{GC})$ works well for {\em Visitall} and {\em Gripper}, where tracking the achieved problem goals provides a monotonic  measure of search progress; however, it becomes insufficient when explicit goals do not provide such information (e.g.~in {\em Baking}, {\em Intrusion}, {\em Lock} or {\em Spanner}).  The benefits of combining our two orthogonal contributions are represented by $PGP(f_{LM})$, where our landmark counting heuristic $f_{LM}$ is used to inform our $PGP$ algorithm. $PGP(f_{LM})$ solves all domains within five minutes (approx.) including the pre-processing time of all the landmark graphs. Figure~\ref{fig:synthesis} shows the generalized plans computed by $PGP(f_{LM})$; the solution for the {\em lock} domain was already shown in Figure~\ref{fig:lock-gplan}. We successfully validated all these solutions on large instances.

Our two orthogonal contributions are also evaluated separately in an {\em ablation study}, where we either ablate the contributed PGP algorithm which is the case for BFS($f_{LM}$), or ablate the contributed heuristic $f_{LM}$ which is the case for PGP($f_{GC}$). BFS($f_{LM}$), has the same coverage as the baseline. Its main drawback is that aggregating landmarks over all input instances may bias search towards certain regions of the space of planning programs where no solution generalized plan exists, e.g. consuming the first program line with an instruction that reaches a new landmark but that must be actually used some lines after invalidating the rest of the program. On the other hand, PGP($f_{GC}$) outperforms the baseline; its heuristic is evaluated only in the subset of $active$ instances, improving the coverage and total time of the baseline, but still it suffers from large plateaus due to the poorly informed heuristic. 

As a rule of thumb, $f_{LM}$ is not better than $f_{GC}$ in domains where the explicit problem goals are already providing an informative notion of search progress; the {\em Gripper} domain is a representative example of this. On the other hand, PGP exploits the fact that, in many domains, computing a succinct solution that generalizes to a small set of instances is generalizing to unseen problems. Therefore, PGP will perform worse than BFS when the programs that successfully solve the problems of the {\em active} set successively fail to generalize to the remaining problems.

\subsubsection{Synthesis with $f_1$ for tie breaking.} In \citeauthor{javi:GP:ICAPS21}~(\citeyear{javi:GP:ICAPS21}), $f_{GC}$ is also used in combination with the structural evaluation function $f_1(\Pi)$, that counts the number of goto instructions in a planning program $\Pi$, and that is used for tie breaking.
In this experiment we evaluate this same tie breaking with our contributions, i.e. $PGP(f_{LM},f_1)$, and compare it with the best original setting in \citeauthor{javi:GP:ICAPS21}~(\citeyear{javi:GP:ICAPS21}), i.e. BFS($f_{GC}$,$f_1$) in our subset of propositional domains. Results in Table~\ref{tab:combined_synthesis} show that  $PGP(f_{LM},f_1)$ also outperforms the {\em state-of-the-art} in {\em GP as heuristic search}, BFS($f_{GC}$,$f_1$), in almost all domains.

\begin{table}[hbtp]
\scriptsize
    \centering
    \resizebox{\columnwidth}{!}{
    \begin{tabular}{|l||c|c|c|c||c|c|c|c|}\hline 
        \multirow{2}{*}{} &  \multicolumn{4}{|c||}{BFS($f_{GC}$,$f_1$)} & \multicolumn{4}{c|}{PGP($f_{LM}$,$f_1$)} \\\cline{2-9}
        & T & M & Ex & Ev & T & M & Ex & Ev \\\hline
        Baking & TE & TE & TE & TE & {\bf 63} & {\bf 35} & {\bf 30K} & {\bf 937K} \\
        Corridor & 49 & {\bf 16} & {\bf 3K} & 67K & {\bf 8} & 43 & {\bf 3K} & {\bf 62K} \\
        Gripper & {\bf 4} & {\bf 14} & {\bf 3K} & {\bf 59K} & 8 & 18 & 19K & 336K \\
        Intrusion & 55 & 183 & 51K & 896K & {\bf 0} & {\bf 18} & {\bf 8} & {\bf 190} \\
        Lock & TE & TE & TE & TE & {\bf 3} & {\bf 45} & {\bf 1K} & {\bf 26K} \\
        Ontable & TE & TE & TE & TE & {\bf 312} & {\bf 97} & {\bf 3K} & {\bf 110K} \\
        Spanner & TE & TE & TE & TE & {\bf 168} & {\bf 604} & {\bf 24K} & {\bf 670K} \\
        Visitall & {\bf 0} & {\bf 7} & {\bf 45} & 496 & {\bf 0} & 17 & 53 & {\bf 458} \\\hline 
    \end{tabular}
    }
    \caption{\small Synthesis with BFS($f_{GC}$,$f_1$) and PGP($f_{LM}$,$f_1$), with the same input settings and metrics from Table~\ref{tab:synthesis}.}
    \label{tab:combined_synthesis}
\end{table}

\section{Related work}
Our GP approach is related to previous work that computes {\em generalized heuristics} for guiding state-space search on new classical planning instances of a given domain~\cite{frances2019generalized,staahlberg2021learning,karia2021learning}. However,  $PGP(f_{LM})$ does not aim learning a generalized heuristic but instead, it leverages the classical planning landmark machinery, that was originally conceived for state space search.  We believe that our approach opens the door to incorporating into GP other successful techniques coming from classical planning, e.g.  {\em helpful actions/preferred operators}.

Most of the previous work on GP compute generalized plans that solve, at once, the entire set of classical planning instances given as input.  $PGP(f_{LM})$ implements a progressive approach that, one by one, processes the full batch of classical planning instances in a GP problem. Remarkably our progressive approach overcomes the main drawback of {\em bottom-up/online} approaches for GP~\cite{Winner03distill:learning,srivastava2011new}, which suffer from the complexity of merging a new individual solution with the previously found solutions.

First-order logic (FOL) policies that specify a strategy for solving planning instances have also shown to generalize to planning domains~\cite{khardon1999learning,martin2004learning}. Validating FOL policies over large instances (with large numbers of objects) is difficult because of the complexity of variable matching; our approach has  constant variable matching complexity since {\em planning programs} have no free variables. On the other hand, generalized polices are able to solve problems when actions are not deterministic~\cite{belle2016foundations}, which connects GP to more general notions of planning, such as MDPs and POMDPs~\cite{kolobov2012planning}.
Unlike related work that focus on the computation of generalized policies, our GP approach does not  require knowing the full state space of the input instances~\cite{frances2021learning}, which may easily be too large to be fully specified, or reformulating actions w.r.t.~a pool of features~\cite{bonet2019learning}.

{\em Deep  Reinforcement  Learning}  (DRL) is also used to learn policies~\cite{sutton2018reinforcement}, represented with {\em Deep Neural Networks} (DNNs), that solve sequential decision-making problems, even when symbolic representations of states and actions are not available~\cite{mnih2015human}. Off-the-shelf tools for learning DNNs have also been successfully applied to learn black-box generalized policies, and heuristics, from PDDL  representations~\cite{bueno2019deep,garg2020symbolic, toyer2020asnets}. DNNs are suitable for black box decision-making, but they are difficult to interpret; DNNs represent knowledge as millions of coupled parameters, so it becomes difficult to identify the piece of knowledge responsible for solving a particular task, as well as to understand whether this piece of knowledge will be useful for unseen problems. Last but not least, {\em model-based DRL} approaches have exhibited good performance in several domains~\cite{hafner2020mastering}, but the learned world models are again represented as NNs; solutions are then difficult to interpret and their generalization capacity in the presence of new objects is not evaluated. 

\section{Conclusions}
We presented $PGP(f_{LM})$, a novel heuristic search approach to GP that progressively processes the classical planning instances of a GP problem, and that leverages a landmark counting heuristic to search in the space of planning programs. $PGP(f_{LM})$ allows to transfer landmark counting heuristics, originally conceived for state-space, to the solution-space search of GP. There is still room for improving our $f_{LM}$ heuristic for GP; the information captured by our landmark graphs could be augmented exploiting cyclic dependencies~\cite{buchner2021exploiting}, considering the remaining number of programmable lines~\cite{Laura:TLandmarks:ICAPS2014}, or leveraging different relaxations of the planning instances~\cite{keyder2010sound}. Besides landmarks,  heuristic planners implement complementary ideas such as {\it helpful actions}/{\it preferred operators}~\cite{hoffmann2001ff}, {\it multi-queue best-first search} for multiple heuristics combination~\cite{Helmert:FD:JAIR06}, or {\it novelty-based exploration}~\cite{lipovetzky2012width}.  A promising future research direction is to incorporate into the {\em GP as heuristic search} approach all these techniques that have proved successful for classical planning.

\section*{Acknowledgements}
\begin{footnotesize}
Javier Segovia-Aguas is supported by TAILOR, a project funded by EU H2020 research and innovation programme no. 952215. Sergio Jim\'enez is supported by Spanish grants RYC-2015-18009 and  TIN2017-88476-C2-1-R. Anders Jonsson is partially supported by Spanish grants PID2019-108141GB-I00. Laura Sebastia is supported by TIN2017-88476-C2-1-R and TAILOR-952215. 
\end{footnotesize}

\bibliography{landmarks}

\clearpage
\appendix
\appendixpage

\section{Experimental Setup and Discussion}

\begin{table*}[t!]
\scriptsize
    \centering
    \resizebox{\textwidth}{!}{
    \begin{tabular}{|l|c||c|c||c|c||c|c||c|c|}\hline 
         \multirow{2}{*}{} & \multirow{2}{*}{$n$, $|Z|$} & \multicolumn{2}{c||}{BFS$(f_{GC})$} & \multicolumn{2}{c||}{BFS($f_{LM})$} & \multicolumn{2}{c||}{PGP$(f_{GC})$}& \multicolumn{2}{c|}{PGP$(f_{LM})$}   \\\cline{3-10}
         & & NE (Dead) & SE & NE (Dead) & SE & NE (Dead) & SE & NE (Dead) & SE \\\hline 
         Baking & 13, 6 & TE & TE & TE & TE & ME & ME & {\bf 1.4M} (1.3M) & {\bf 180K} \\
         Corridor & 11, 2 & 120K (33K) & 876K & {\bf 114K} (102K) & {\bf 112K} & 120K (31K) & 181K & 121K (57K) & 129K  \\
         Gripper & 8, 4 & {\bf 63K} (16K) & 461K & 339K (317K) & 223K & {\bf 63K} (16K) & 46K & 336K (314K) & {\bf 22K} \\
         Intrusion & 9, 1 & 1.3M (207K) & 11.3M & {\bf 190} (27) & 1.7K & 1.3M (204K) & 3.3M & {\bf 190} (26) & {\bf 480}   \\ 
         Lock & 12, 2 & TE & TE & TE & TE & ME & ME & {\bf 30K} (4.6K) & {\bf 26K} \\
         Ontable & 11, 3 & TE & TE & TE & TE & 366K (54K) & 312K & {\bf 113K} (38K)& {\bf 75K} \\
         Spanner & 12, 5 & TE & TE & TE & TE & TE & TE  & {\bf 705K} (674K)& {\bf 63K} \\
         Visitall & 7, 2  & 511 (380) & 1.3K & {\bf 249} (154) & 970 & 900 (669) & 581 & 446 (294) & {\bf 335}   \\\hline
    \end{tabular}
    }
    \caption{\small Number of program lines $n$ and pointers $|Z|$. NE stands for nodes evaluated, Dead for nodes marked as deadends, SE for states evaluated, TE and ME for time and memory exceeded, respectively. K is $10^3$ and M is $10^6$. Best results in bold.}
    \label{tab:evaluations}
\end{table*}

Table~\ref{tab:synthesis} focus on reporting data about nodes in the search where every node is a (partial) program. The evaluation of a node consists of first applying it on every input instance, then retrieve the resulting states of each instance, and compute the evaluation functions over them to insert the node in the open queue. Since $BFS$ works on the full set of instances and $PGP$ only on the active set, one interesting observation is how many states are both evaluating. Note that node evaluation may return a failure, e.g. the node is a deadend, in such cases the node evaluation counter increases by one but the state evaluation counter remains unchanged. Table~\ref{tab:evaluations} shows these differences between node evaluations, node deadends and state evaluations for each setting from Table~\ref{tab:synthesis}. From Table~\ref{tab:evaluations} we conclude that $f_{LM}$ almost dominates in all facets of evaluations in the search, being Gripper the only domain where $f_{GC}$ does better in the number of node evaluations; and the best general performance and coverage is found in the $PGP(f_{LM})$ setting, complementing Table~\ref{tab:synthesis} results.

Next we analyze each computed solution by $PGP(f_{LM})$ in Table~\ref{tab:synthesis}, which all have been successfully validated and generalization can be proven \footnote{The experiments are all reproducible by following the code tutorial.}.

\subsection{Baking}
This domain consists of baking cakes in an oven after mixing eggs and flour in a pan, which has to be cleaned before baking the next cake. There are $7$ action schemes, to put an egg or flour in a pan, mix the egg and flour in the pan, put and remove a pan from an oven, bake a cake or clean a pan.\\

\noindent\textit{Inputs:} there are $10$ training instances to bake $n\in[1,10]$ cakes. In the validation set there are $n\in[11,60]$ cakes (a total of $50$ instances).\\

\noindent\textit{Baking solution:} the parameters are pointers to indirectly address typed objects eggs ($z_e$), flour ($z_f$), pan ($z_p$), oven ($z_o$), cake ($z_c$), and soap ($z_s$).
    \begin{lstlisting}[mathescape]
0. putegginpan($z_e$,$z_p$)
1. putflourinpan($z_f$,$z_p$)
2. mix($z_e$,$z_f$,$z_p$)
3. putpaninoven($z_p$,$z_o$)
4. bakecake($z_o$,$z_p$,$z_c$)
5. removepanfromoven($z_p$,$z_o$)
6. cleanpan($z_p$,$z_s$)
7. inc($z_c$)
8. inc($z_e$)
9. inc($z_f$)
10. inc($z_s$)
11. goto(0,$\neg y_z$)
12. end 
    \end{lstlisting} 

\noindent\textit{The baking solution generalizes.} 
\begin{proof}
All pointers start indexing the first typed object of their list. The program puts the first egg and flour in the pan, mix them, then it puts the pan in the oven, bakes the cake, removes the pan from the oven and leaves the pan clean. Then, the cake, egg, flour and soap pointers are increased by one, and the procedure repeats until the soap pointer cannot be increased more. Since, there are as many soap objects as eggs, flours and cakes, this guarantees that all cakes are going to be baked iteratively.
\end{proof}

\subsection{Blocks Ontable}
This domain is a specific instance of the blocksworld domain, where $n$ blocks are arranged into one or multiple towers of blocks with the goal of putting all down onto the table. There are $4$ action schemes to pick up a block from the table, or put it down onto the table, and to stack or unstack a block onto/from another one.\\

\noindent\textit{Inputs:} the training set consists of $12$ instances with $n\in[10,15]$ blocks, and the validation set has $50$ instances with $n\in[16,65]$ blocks.\\

\noindent\textit{Blocks ontable solution:} there are $3$ pointers ($z_{o1}$, $z_{o2}$ and $z_{o3}$) to indirectly address the blocks of each instance.
    \begin{lstlisting}[mathescape]    
0. unstack($z_{o1}$,$z_{o2}$)
1. inc($z_{o2}$)
2. goto(0,$\neg y_z$)
3. put-down($z_{o1}$)
4. clear($z_{o2}$)
5. inc($z_{o1}$)
6. goto(0,$\neg y_z$)
7. clear($z_{o1}$)
8. inc($z_{o3}$)
9. goto(0,$\neg y_z$)
10. end
    \end{lstlisting}

\noindent\textit{The blocks ontable solution generalizes.} 
\begin{proof}
First, it tries to unstack the first object $z_{o1}$ from any other object $z_{o2}$ and put it down onto the table. Then, it resets $z_{o2}$ and increase $z_{o1}$ by one, and repeats until it has tried all unstacks and putdowns between two objects. In the worst case, there is a single tower in reverse order, an every time that Lines 0-6 are executed, only one block is put down onto the table. Since the number of blocks onto the table are increased at least by one when applying Lines 0-6, executing that partial program $n$ times (where $n$ is the number of blocks) guarantees to have all blocks onto the table, which is what Lines 7-9 do.
\end{proof}

\subsection{Corridor}
In this domain there is only one action schema to move between two adjacent locations of a corridor. The initial state and goal are two arbitrary locations, and requires an extra predicate in the initial state that is a copy of the ``at'' in the goal for generalization. \\

\noindent\textit{Inputs:} consist of corridors of length $n$, with $10$ training instances where $n\in[5,14]$, and $50$ validation instances where $n\in[12,61]$. \\

\noindent\textit{Corridor solution:} pointers $z_1$ and $z_2$ are used to indirectly address location objects.
    \begin{lstlisting}[mathescape]  
0. inc($z_1$)
1. move($z_2$,$z_1$)
2. inc($z_1$)
3. inc($z_2$)
4. goto(1,$\neg y_z$)
5. move($z_1$,$z_2$)
6. set($z_1$,$z_2$)
7. dec($z_2$)
8. test(goal-at($z_1$))
9. goto(4,$y_z$)
10. end
    \end{lstlisting} 

\noindent\textit{The corridor solution generalizes.} 
\begin{proof}
Line 0 is applied only once, setting $z_2$ and $z_1$ to the first and second locations, respectively. After, it moves from $z_2$ to $z_1$ (right move), increases both pointers, and repeats until the agent is at the rightmost location. Then, the program will finish with Lines 5-9 if the agent is at the goal location, otherwise it will repeat the procedure of moving one location to the left until the goal location is reached, applying in the worst case twice as many planning actions as the size of the corridor.
\end{proof}

\subsection{Gripper}
There are 3 action schemes: to move between the two rooms, to pick a ball with a gripper from a location, and to drop a ball in a location with the gripper that holds the ball. The goal is to move all balls from room A to room B by using left/right grippers.\\

\noindent\textit{Inputs:} the training set has $10$ instances with $n\in[2,11]$ balls, always starting in room A. The validation set has $50$ instances with $n\in[12,61]$ balls.\\

\noindent\textit{Gripper solution:} the room pointers are $z_{r1}$ and $z_{r2}$, the ball pointer is $z_b$ and the gripper pointer is $z_g$.
    \begin{lstlisting}[mathescape]    
0. pick($z_b$,$z_{r1}$,$z_g$)
1. inc($z_{r2}$)
2. move($z_{r1}$,$z_{r2}$)
3. drop($z_b$,$z_{r2}$,$z_g$)
4. move($z_{r2}$,$z_{r1}$)
5. inc($z_b$)
6. goto(0,$\neg y_z$)
7. end 
    \end{lstlisting}

\noindent\textit{The gripper solution generalizes.} 
\begin{proof}
The sequence of actions for one ball (Lines 0-3) is to pick up the first ball $z_b$ with the first gripper $z_g$ (left) at room $z_{r1}$ (room A), to increase the pointer $z_{r2}$ to point to room B, to move from room A to room B, and to drop the ball $z_b$ that $z_g$ holds at room B. Then, to go for the next ball (Lines 4-6), it moves back to room A, and increase the pointer $z_b$ by one until there are no more balls.  Note that $z_{r2}$ is increased in every loop but it is applicable only once, thus after the second step it always points to room B while $z_{r1}$ always points room A. Since this procedure works for one ball, leaves the robot at room A, and runs iteratively for each ball, generalizing to any number of balls.
\end{proof}

\subsection{Intrusion}
In this domain there is 9 action schemes related to a cyber-security intrusion, all parameterized for one specific host. It can recon a host, gather information and obtain access, modify the files of the system, clean the logs, vandalize the host, gain root access, download files and steal data. There could be multiple goals, but we focus on stealing the data of all hosts, which makes some actions to be unhelpful for generalization. \\

\noindent\textit{Inputs:} there are $10$ training instances with  $n\in[1,10]$ hosts, and a validation set of $50$ instances with $n\in[11,60]$ hosts. \\

\noindent\textit{Intrusion solution:} pointer $z_h$ addresses the current host.
    \begin{lstlisting}[mathescape]
0. recon($z_h$)
1. break-into($z_h$)
2. clean($z_h$)
3. gain-root($z_h$)
4. download-files($z_h$)
5. steal-data($z_h$)
6. inc($z_h$)
7. goto(0,$\neg y_z$)
8. end     
    \end{lstlisting}

\noindent\textit{The intrusion solution generalizes.} 
\begin{proof}
The steps to steal data from a host $z_h$ are shown in Lines 0-5, once the data is stolen for $z_h$, it goes for the next host increasing $z_h$ by one, and going back to the first line until there are no more hosts. This procedure steals the data of all hosts, no matter the number of hosts.
\end{proof}

\subsection{Lock}
This is the example domain, where an agent starts in an arbitrary location of a corridor, then it has to pick up a key in the rightmost location,  and move back to the leftmost location to open a lock with the key. The domain has one action schema to move between locations, and 3 for pick up a key from a location, drop the key in a given location, and open a lock with the key.\\

\noindent\textit{Inputs:} the training set has $10$ instances with corridors of length $n\in[5,14]$, where the lock and the key are always in the leftmost and rightmost cells, respectively. The validation set has $50$ instances with $n\in[12,61]$.\\

\noindent\textit{Lock solution:} pointers $z_1$ and $z_2$ are used to indirectly address locations of the corridor.
    \begin{lstlisting}[mathescape]
0. inc($z_1$)
1. move($z_2$,$z_1$)
2. inc($z_1$)
3. inc($z_2$)
4. goto(1,!$y_z$)
5. pickup($z_1$)
6. dec($z_1$)
7. move($z_1$,$z_2$)
8. dec($z_2$)
9. goto(5,!$y_z$)
10. open($z_1$)
11. end
    \end{lstlisting}

\noindent\textit{The lock solution generalizes} 
\begin{proof}
Lines 0-4 moves the agent to the rightmost location, then it picks up the key (Line 5), and in Lines 6-9 it moves to the leftmost location (note that it tries to pick up the key in every location when going to the left, but only the first one is applicable). Once in the leftmost location, it opens the lock (Line 10). This applies no matter the size of the corridor.
\end{proof}

\subsection{Spanner}
In this domain a man starts in the leftmost location of a corridor, and must pick all spanners on its way to the rightmost location, and tighten all nuts with the collected spanners. Two of the main difficulties are that spanners can only be used once and the man can only traverse the corridor in one direction, so it is a deadend whenever the man does not pick a spanner in a location and moves to the next location. There are 3 action schemas in this domain, a man can walk to the next adjacent location, pick up spanners and tighten nuts with spanners. \\

\noindent\textit{Inputs:} consists of $10$ training instances of corridors of length $n \in [1,10]$, and $2\times n$ spanners and nuts. The validation set has $50$ instances with $n\in[12,61]$.\\

\noindent\textit{Spanner solution:} the location pointers are $z_{l1}$ and $z_{l2}$, for spanners there is a pointer $z_s$, for nuts a pointer $z_n$ and for the man is the pointer $z_m$.
    \begin{lstlisting}[mathescape]    
0. pickup_spanner($z_{l1}$,$z_s$,$z_m$)
1. tighten_nut($z_{l1}$,$z_s$,$z_m$,$z_n$)
2. inc($z_n$)
3. inc($z_s$)
4. goto(0,$\neg y_z$)
5. inc($z_{l1}$)
6. walk($z_{l2}$,$z_{l1}$,$z_m$)
7. clear($z_n$)
8. clear($z_s$)
9. inc($z_{l1}$)
10. goto(0,$\neg y_z$)
11. end
    \end{lstlisting}

\noindent\textit{The spanner solution generalizes.} 
\begin{proof}
First, there are as many spanners as nuts, so in Lines 0-4 the man tries to pick up a spanner and use it to tighten a nut (note that inapplicable tightens make the man to keep the spanners), and continues with the next spanner and nut, until pointer $z_s$ visits all spanners. Then the man moves to the next adjacent location and resets $z_s$ and $z_n$ (Lines 5-8), and continues in the first line until $z_{l1}$ can not be increased more (Lines 9-10). In summary, in every location $z_{l1}$ the man picks up all spanners from that location, moves to the next location and use the available spanners to tighten all loose nuts when applicable. Thus, iff the initial state is not a dead-end, it generalizes no matter the distribution of spanners and nuts or the length of the corridor.
\end{proof}

\subsection{Visitall}
Starting from the bottom-left corner, an agent must visit all cells of an $n\times n$ grid. There is only one action schema  ``visit'' to mark a cell as visited. \\

\noindent\textit{Inputs:} there are $10$ training instances of $n\times n$ grid sizes such that $n\in[2,11]$. In the validation, there are $50$ instances where $n\in[12,61]$. \\

\noindent\textit{Visitall solution:} pointer $z_i$ is used for row objects, and $z_j$ for column objects.
    \begin{lstlisting}[mathescape]    
0. visit($z_i$,$z_j$)
1. inc($z_i$)
2. goto(0,$\neg y_z $)
3. clear($z_i$)
4. inc($z_j$)
5. goto(0,$\neg y_z$)
6. end
    \end{lstlisting}

\noindent\textit{The visitall solution generalizes.} 
\begin{proof}
Location $(0,0)$ is visited when applying the first instruction (Line 0), then it increases $z_i$ by one and repeats until there are no more rows for $z_j=0$, so it visits all cells of the first column (Lines 0-2). After that, it goes back to the first row, moves to the next column, and continues from the first program line until there are no more columns (Lines 3-5). This program visits all the cells in the grid, row by row, so it generalizes no matter the size of the grid. 
\end{proof}

\section{Evaluation of PGP($f_{\left\lVert LM \right\rVert}$)}
One of the main drawbacks of $f_{LM}$ is the searching bias towards programs that reach more landmarks for certain instances. One way to mitigate this effect is to normalize the weight of each instance by its total number of landmarks. We call this the {\em normalized landmarks} heuristic, which works with a precision of $4$ digits. However, the outcomes of this new normalized heuristic is exactly the same as $f_{LM}$ in the given set of generalized planning problems, which means that even though $f_{LM}$ might be biased from input data, this is not observed in the randomly generated datasets.

\end{document}